\PassOptionsToPackage{dvipsnames,svgnames}{xcolor}
\documentclass[sigconf,nonacm]{acmart}
\AtBeginDocument{%
  }

\usepackage[utf8]{inputenc}
\usepackage[T1]{fontenc}
\usepackage{amsmath}
\usepackage{amsthm}
\usepackage{amsfonts}
\usepackage{booktabs}
\usepackage{multirow}
\usepackage{adjustbox}
\usepackage{rotating}
\usepackage{thm-restate}
\usepackage{enumitem}
\usepackage{subcaption}
\usepackage[linesnumbered, ruled, vlined]{algorithm2e}
\DontPrintSemicolon
\SetKwInOut{Input}{Input}
\SetKwInOut{Output}{Output}
\SetKwProg{function}{Function}{:}{}
\SetKwProg{parfor}{parallel for}{\:do}{}

\SetVlineSkip{0.2pt}        
\SetInd{0.7em}{0.7em}     

\usepackage{graphicx}
\usepackage{enumitem}
\usepackage{textgreek}
\usepackage{tabularx}
\usepackage{xspace}

\definecolor{ForestGreen}{rgb}{0.0333,0.4451,0.0333}
\definecolor{DarkRed}{rgb}{0.65,0,0}
\definecolor{Red}{rgb}{1,0,0}
\hypersetup{
    unicode=false,          
    colorlinks=true,        
    linkcolor=BrickRed,          
    citecolor=OliveGreen,        
    filecolor=magenta,      
    urlcolor=cyan           
}

\newtheorem{theorem}{Theorem}

\newtheorem{definition}{Definition}
\newtheorem{observation}{Observation}

\makeatletter
\newcommand{\multiline}[1]{%
  \begin{tabularx}{\dimexpr\linewidth-\ALG@thistlm}[t]{@{}X@{}}
    #1
  \end{tabularx}
}
\makeatother

\newcommand{\myparagraph}[1]{\vspace{.03in}\noindent {\bf #1.}}

\DeclareMathOperator*{\argmin}{arg\,min}

\newcommand{\R}{\ensuremath{\mathbb{R}}}
\newcommand{\Ch}{\ensuremath{\mathtt{Ch}}\xspace}
\newcommand{\best}[1]{{\color{Green}\underline{\textbf{#1}}}}
\newcommand{\better}[1]{{\color{Blue} \textbf{#1}}}

\def \eps {\varepsilon\xspace}

\def \cC {\mathcal{C}\xspace}

\def \cL {\mathcal{L}\xspace}

\newcommand{\NN}{{\sf NN}\xspace}
\newcommand{\hacnn}{{\sf HAC-NN}\xspace}
\newcommand{\merge}{{\sf Merge}\xspace}

\newcommand{\emp}[1]{\emph{\textbf{#1}}} 

\begin{document}

\title{Chamfer-Linkage for Hierarchical Agglomerative Clustering}

\author{Kishen N Gowda}
\affiliation{%
  \institution{University of Maryland}
  \city{College Park}
  \country{USA}
}

\author{Willem Fletcher}
\affiliation{%
  \institution{Brown University}
  \city{Providence}
  \country{USA}
}

\author{MohammadHossein Bateni}
\affiliation{%
  \institution{Google Research}
  \city{New York}
  \country{USA}
}

\author{Laxman Dhulipala}
\affiliation{%
  \institution{University of Maryland}
  \city{College Park}
  \country{USA}
}

\author{D Ellis Hershkowitz}
\affiliation{%
  \institution{Brown University}
  \city{Providence}
  \country{USA}
}

\author{Rajesh Jayaram}
\affiliation{%
  \institution{Google Research}
  \city{New York}
  \country{USA}
}

\author{Jakub Łącki}
\affiliation{%
  \institution{Google Research}
  \city{New York}
  \country{USA}
}

\begin{abstract}
  Hierarchical Agglomerative Clustering (HAC) is a widely-used clustering method based on repeatedly merging the closest pair of clusters, where inter-cluster distances are determined by a linkage function. Unlike many clustering methods, HAC does not optimize a single explicit global objective; clustering quality is therefore primarily evaluated empirically, and the choice of linkage function plays a crucial role in practice. However, popular classical linkages, such as single-linkage, average-linkage and Ward's method show high variability across real-world datasets and do not consistently produce high-quality clusterings in practice. 
  
  In this paper, we propose \emph{Chamfer-linkage}, a novel linkage function that measures the distance between clusters using the Chamfer distance, a popular notion of distance between point-clouds in machine learning and computer vision. We argue that Chamfer-linkage satisfies desirable concept representation properties that other popular measures struggle to satisfy. Theoretically, we show that Chamfer-linkage HAC can be implemented in $O(n^2)$ time, matching the efficiency of classical linkage functions. Experimentally, we find that Chamfer-linkage consistently yields higher-quality clusterings than classical linkages such as average-linkage and Ward's method across a diverse collection of datasets. Our results establish Chamfer-linkage as a practical drop-in replacement for classical linkage functions, broadening the toolkit for hierarchical clustering in both theory and practice.
\end{abstract}


\maketitle

\section{Introduction}\label{sec:intro}

Hierarchical Agglomerative Clustering (HAC) is a fundamental unsupervised learning method that is widely used due to its ability to produce high quality clusterings that reveal multi-scale structure in data.
The algorithm follows a simple, greedy, bottom-up procedure: starting with each data point as a singleton cluster, it iteratively merges the two \emph{closest} clusters until only a single cluster remains, thereby producing a \emph{dendrogram} that represents a hierarchy of groupings. 
Since its introduction nearly a century ago, HAC has seen significant work both on the theoretical~\cite{schlee1975numerical, lance1967general, king1967step, tsenghac} and practical~\cite{dhulipala2022hierarchical, dhulipala2021hierarchical, dhulipala2023terahac} side, is included in a wide variety of data science libraries including SciPy~\cite{virtanen2020scipy}, scikit-learn~\cite{pedregosa2011scikit}, and fastcluster~\cite{mullner2013fastcluster}, and has been widely applied throughout the sciences, e.g., in computational biology~\cite{hua2017mgupgma, stefan1996multiple}, document clustering~\cite{zhao2002evaluation}, social network analysis~\cite{blundell2013bayesian}, and many other application areas.

When applying the merging procedure in HAC, a fundamental question is how to measure the closeness of clusters. Cluster closeness is quantified by a ``linkage function'' and classical choices of linkage function include \emph{single-linkage} (inter-cluster distance is the distance between the closest points in the clusters), \emph{complete-linkage} (distance is the distance between the farthest points in the clusters), \emph{average-linkage} (the distance is the average distance between points in both clusters), and \emph{Ward's linkage} (the distance is proportional to the increase in the sum-of-squares costs to the centroid of the merged cluster).

Unlike many clustering methods, HAC does not optimize a single explicit global objective; instead, its effectiveness depends critically on the choice of linkage function and is therefore evaluated empirically.
Although these classical linkage functions are widely used and frequently recommended to practitioners---e.g., in the aforementioned data science toolkits---their clustering quality on real-world datasets is highly variable.
For example, our empirical evaluation indicates that no existing classical linkage function is able to consistently achieve within \textbf{29\%} of the best clustering quality (measured in terms of \emph{Adjusted Rand Index (ARI)}) achieved by any classical linkage across all 18 diverse datasets we study; see Section~\ref{sec:qual_evals}.
Given this situation, a natural question is: 

\begin{center}
\emp{What is the right linkage function for HAC?}
\end{center}

Ideally, a linkage function would admit certain natural properties that make its clustering decisions more intuitive, and potentially also improves quality.
One natural desideratum is producing \emph{balanced dendrograms}, since long chains of cluster merges are unusual in real-world datasets.
Another desirable property is capturing \emph{concept representation}, i.e., the inter-cluster distances should faithfully capture whether the ``concepts'' (i.e., the points) present in one cluster are well-represented within the other.
This is an attractive property since, for instance, when merging clusters $A$ and $B$ in a hierarchical clustering, the cluster being merged into (say $B$) is rarely a single cluster---it contains a number of related sub-clusters. Instead of looking at how similar $A$ is to \emph{all} concepts in $B$, it's sufficient for $A$ to be close to one of $B$'s sub-clusters.

Overly simple schemes like single- or complete-linkage myopically focus on only the nearest or furthest inter-cluster distances, and can thus miss cluster merges that satisfy concept representation, or perform merges that significantly violate it.
On the other hand, average-linkage considers all-pairs distances which can lead to over-weighting the lack of concept representation for far-away (noise) points.
This raises the following question: 

\begin{center}
    \emp{Can we design a linkage function that captures concept representation, while also performing consistently well across real-world datasets?}
\end{center}

In this paper, we propose and study \emph{Chamfer-linkage}, which measures the distance between clusters as 
\begin{align*}
    \Ch(A,B):= \sum_{a\in A} \min_{b\in B} d(a,b),
\end{align*}
where $d(a,b)$ is the underlying distance metric, e.g., the Euclidean distance.
In plain words, $\Ch(A,B)$ sums for each point in $A$, the distance to the closest point in $B$ (note that $\Ch(A,B)$ can be different than $\Ch(B, A)$, in contrast to the symmetry of the other linkage functions\footnote{Note that one can naturally symmetrize Chamfer-linkage by considering $\min\{\Ch(A,B),\Ch(B,A)\}$ as the distance between clusters $A$ and $B$. However, this yields the same merge decisions as viewing HAC as minimizing $\Ch(\cdot,\cdot)$ over ordered pairs. We nevertheless use the asymmetric viewpoint since it provides the cleanest intuition for concept representation and for how merge decisions are guided.}). 
Chamfer distance has a rich history in computer vision and machine learning as a measure of dissimilarity between point clouds~\cite{kusner2015word, wan2019transductive, sudderth2004visual, fan2017point, jiang2018gal} where it is often used as a fast proxy for the more computationally expensive Earth-Mover distance~\cite{kusner2015word}.
More recently, Chamfer distance has served as the basis for state-of-the-art multivector embedding models such as ColBERT, ColPali, and related models~\cite{khattab2020colbert, jayaram2024muvera, faysse2024colpali, santhanam2022plaid}.
Chamfer distance is thus a natural linkage function, but surprisingly, Chamfer-linkage HAC appears unstudied in prior work.

A nice property of Chamfer-linkage is that it naturally satisfies concept representation.
In particular, $\Ch(A,B)$ being small captures the fact that for each point in $A$, there exists a nearby representative in $B$, so that the concepts in $A$ can be matched to concepts present in $B$ with low total cost.
On the other hand, worrying about how spread out the cluster $A\cup B$ will become may not be as important, but this is the consideration made by average-linkage and other classical linkages like complete-linkage.
Furthermore, the underlying asymmetry in Chamfer tends to prioritize merges where one cluster is well-represented within another, which discourages small clusters from persisting until late in the process and can lead to more balanced dendrograms in practice (which we observe empirically in Section~\ref{sec:experiments}).

In this work, we show that Chamfer-linkage’s emphasis on concept representation yields strong and consistent empirical performance in practice, while also being efficiently computable.

\myparagraph{Theoretical Results}
It is not difficult to give a naive implementation of Chamfer-linkage HAC that runs in $O(n^3)$ time, and looking at the definition of $\Ch(A,B)$ one may reasonably wonder if this is the best possible.
We show that Chamfer-linkage satisfies certain properties, which, when used in conjunction with carefully chosen data structures, allow us to compute Chamfer-linkage HAC significantly faster, in $O(n^2)$ time.
Thus, our result shows that Chamfer-linkage HAC can be computed as quickly as any of the classical linkage functions (most of which admit $O(n^2)$ time algorithms using the nearest-neighbor chain method~\cite{benzecri, irbook}). 
Moreover, this quadratic running time is essentially the best possible for exact HAC under standard fine-grained complexity assumptions~\cite{abboudhac}.
We also show a novel space-time tradeoff for Chamfer-linkage HAC, showing that we can reduce the space usage to $O(n^2 / t)$ at the cost of increasing the running time to $O(n^2 t)$, for any parameter $t \in [1, n]$.

\myparagraph{Empirical Results}
Finally, we design a fast C++ implementation of Chamfer-linkage HAC and compare our implementation with implementations of five classical linkages, as well as three other natural variations on Chamfer-linkage HAC.
Along the way, we also produce faster implementations of these classical linkage functions, which significantly speed up existing (optimized) versions in popular packages like {\tt fastcluster}~\cite{mullner2013fastcluster} and {\tt scikit-learn}~\cite{pedregosa2011scikit} by up to \textbf{5.75--9.28}{\boldmath$\times$} respectively.
Our experiments show that Chamfer-linkage consistently outperforms all classical linkage functions in clustering quality---achieving up to \textbf{57\%} improvement in \emph{Adjusted Rand Index (ARI)} scores over the best baseline, \textbf{6\%} on average, and never more than \textbf{8\%} worse---while also producing well-balanced dendrograms with low height.
We will open-source all of our new implementations (along with easy-to-use Python bindings) upon publication.

In summary, our contributions are:
\begin{itemize}[topsep=0pt,itemsep=0pt,parsep=0pt,leftmargin=15pt]
\item We introduce \emph{Chamfer-linkage}, a new linkage function for HAC motivated by concept representation and balanced hierarchies.
\item We give an exact $O(n^2)$-time algorithm for Chamfer-linkage HAC (matching the best possible worst-case guarantees for exact HAC under standard fine-grained assumptions), along with a space--time tradeoff that uses $O(n^2/t)$ space and $O(n^2 t)$ time.
\item We provide fast implementations and show empirically that Chamfer-linkage achieves consistently strong clustering quality while producing well-balanced dendrograms in practice.
\end{itemize}

\section{Preliminaries}\label{sec:prelims}

\begin{figure}[t]
\vspace{-2em}
    \centering
    \begin{subfigure}[b]{0.19\textwidth}
        \centering
        \includegraphics[width=\textwidth, trim=0mm 20mm 140mm 10mm, clip]{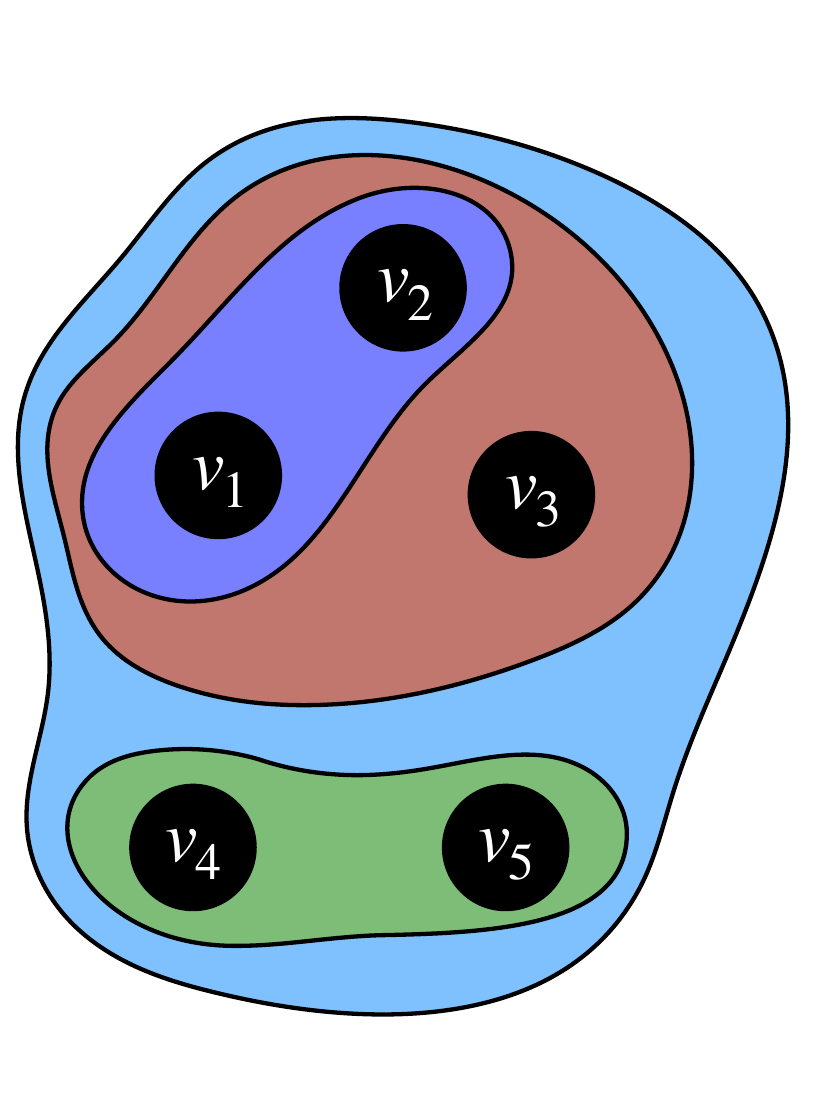}
        \caption{HAC Clustering.}\label{sfig:hac2}
    \end{subfigure}  \hspace{3.5em}
    \begin{subfigure}[b]{0.19\textwidth}
        \centering
        \includegraphics[width=\textwidth,trim=0mm 20mm 125mm 10mm, clip]{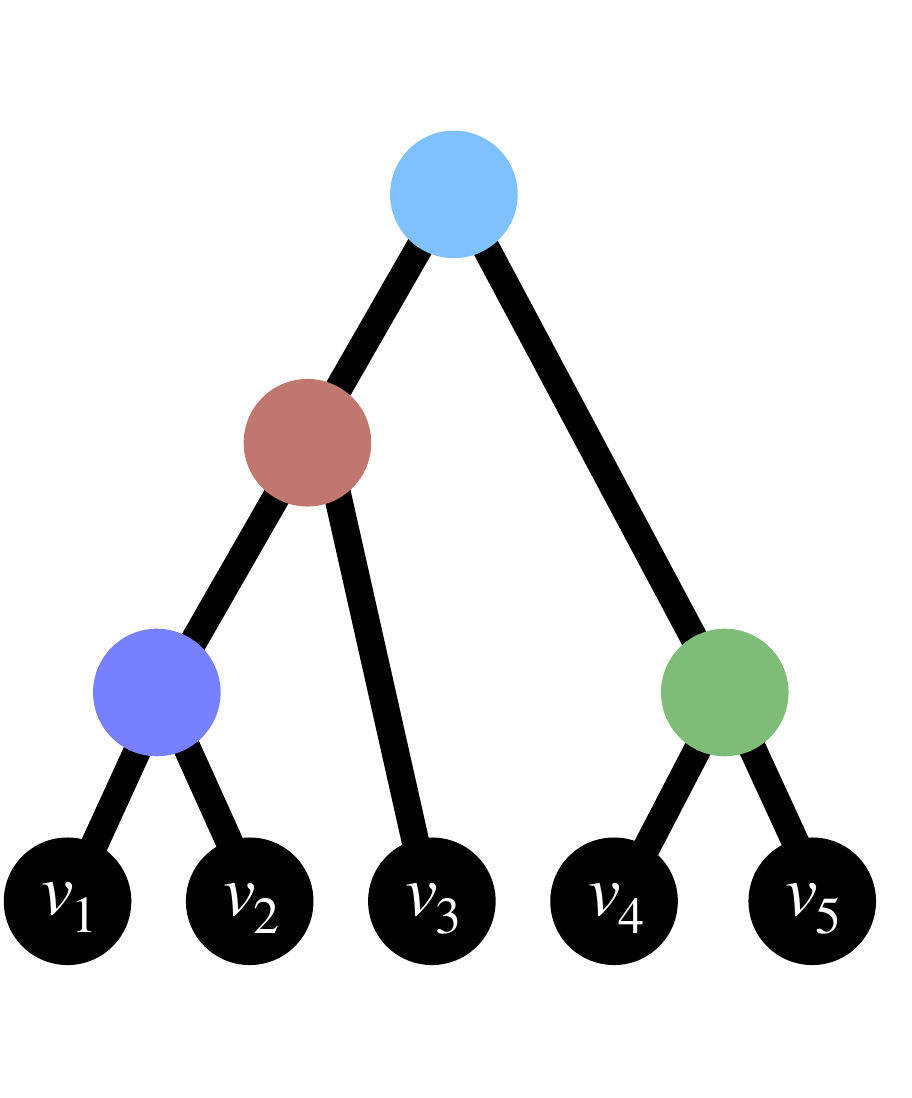}
        \caption{Dendrogram.}\label{sfig:hac3}
    \end{subfigure}   
    \vspace{-0.25em}
    \caption{An example of HAC.} \label{fig:HAC}
\end{figure}

\myparagraph{Hierarchical Agglomerative Clustering (HAC)}
Given a set of points $P \subset \R^d$, and a linkage function $\cL$, HAC begins by placing each point in its own cluster and maintaining a set of current clusters $\cC$; we refer to them as \emph{active} clusters. In each iteration, the algorithm selects the pair of active clusters $A,B \in \cC$ that minimizes the value $\cL(A,B)$, merges them, and updates $\cC\gets \cC\setminus \{A,B\}\cup \{A\cup B\}$. This process continues until only a single cluster remains.

The output of HAC is a \emph{dendrogram}: a rooted binary tree representing the sequence of merges. The leaves correspond to the input points, and each internal node corresponds to the cluster formed by merging the two clusters represented by its children during the execution. See Figure~\ref{fig:HAC}.

HAC can be implemented in various ways. Two widely used methods are the nearest-neighbor chain algorithm~\cite{benzecri, irbook} and the heap-based approach~\cite{irbook}. In this paper, we use a more direct implementation that we refer to as \hacnn. In this approach, each active cluster explicitly maintains its nearest neighbor. In each iteration, the algorithm queries the nearest neighbor of every active cluster, merges the best pair, and updates the nearest neighbors of all clusters that have their nearest-neighbor change.

Classical linkage functions like single-, complete-, and average-linkage can be implemented in the \hacnn framework in $O(n^2\log n)$ time and $O(n^2)$ space using \emph{augmented maps}. Augmented maps are ordered map data structures supporting insertions and deletions in $O(\log n)$ time, and also allow range queries using auxiliary \emph{augmented values}. In our setting, we use augmented values to maintain the minimum-valued entry, which enables retrieval of the minimum, i.e., nearest neighbor, in $O(1)$ time. See~\cite{Sun2018PAM} for more details about augmented maps. See Algorithm~\ref{alg:HAC_NN} for the full description; the \merge subroutine here encapsulates linkage-specific logic and the updates to the nearest neighbors. 

In Section~\ref{sec:quadratic-time}, we show that Chamfer-linkage can be implemented in $O(n^2)$ time using the \hacnn framework.

\IncMargin{1.2em}
\begin{algorithm}[t]
\caption{The \hacnn framework.}\label{alg:HAC_NN}
\Input{Set of points $P\subset \R^d$}
\Output{Dendrogram}
\function{\hacnn}{
    $\cC \gets \{\{p\}:p\in P\}$\\
    \For{$U \in \cC$}{
        $\NN(U) \gets \argmin_{V\in \cC\setminus \{U\}} \cL(U,V)$\\
    }
    \While{$|\cC|>1$}{ 
        $A \gets \argmin_{X\in \cC}\cL(X,\NN(X))$\\
        $B \gets \NN(A)$\\
        $\merge(A,B,\NN)$\\
        $\cC \gets \cC\setminus \{A,B\}\cup \{A\cup B\}$
    }
}
\end{algorithm}
\DecMargin{1.2em}

\myparagraph{Chamfer distance}
Next, we define the Chamfer distance, and some variants that we study in this paper.

\begin{definition}[Chamfer Distance]
Let $A$ and $B$ denote two sets of points, where distance between individual points is given by a metric $d$. Then, the (asymmetric) Chamfer distance from $A$ to $B$ is defined as,
\begin{align*}
    \Ch(A,B):= \sum_{a\in A} \min_{b\in B} d(a,b).
\end{align*}
\end{definition}

We also consider the following natural variants of the Chamfer distance in this paper:
\begin{itemize}[topsep=0pt,itemsep=0pt,parsep=0pt,leftmargin=15pt]
    \item Normalized Chamfer distance: $$\Ch_N(A,B) := \Ch(A,B)/|A|,$$
    \item Symmetric Chamfer distance: $$\Ch_S(A,B) := \Ch(A,B) + \Ch(B,A),$$
    \item Symmetric and Normalized: 
    $$\Ch_{NS}(A,B) := \Ch_N(A,B) + \Ch_N(B,A).$$
\end{itemize}
It is easy to verify that none of the variants of the Chamfer distance satisfy the triangle inequality, and are thus not a metric.
We note that the aforementioned asymmetric distances ($\Ch$ and $\Ch_N$) when used as a linkage function is equivalent to the symmetric function that combines the two directions using $\min$, i.e., $\min(\Ch(A,B), \Ch(B, A))$. We nevertheless use the asymmetric viewpoint throughout the paper since it provides the cleanest intuition for our algorithmic properties and update rules.

\myparagraph{Properties of Chamfer-linkage}
Most linkage functions $\cL$ satisfy the useful property of \emph{reducibility},
i.e., for any clusters $A,B,C$, we have that $\cL(A, B \cup C) \geq \min(\cL(A, B), \cL(A,C))$.
Single-, complete-, average-, and Ward's linkages satisfy reducibility, while centroid linkage is not a reducible linkage.

\begin{observation}
    None of the variants of Chamfer-linkage satisfy reducibility.
\end{observation}

\begin{figure}[t]
    \centering
    \includegraphics[width=0.28\textwidth]{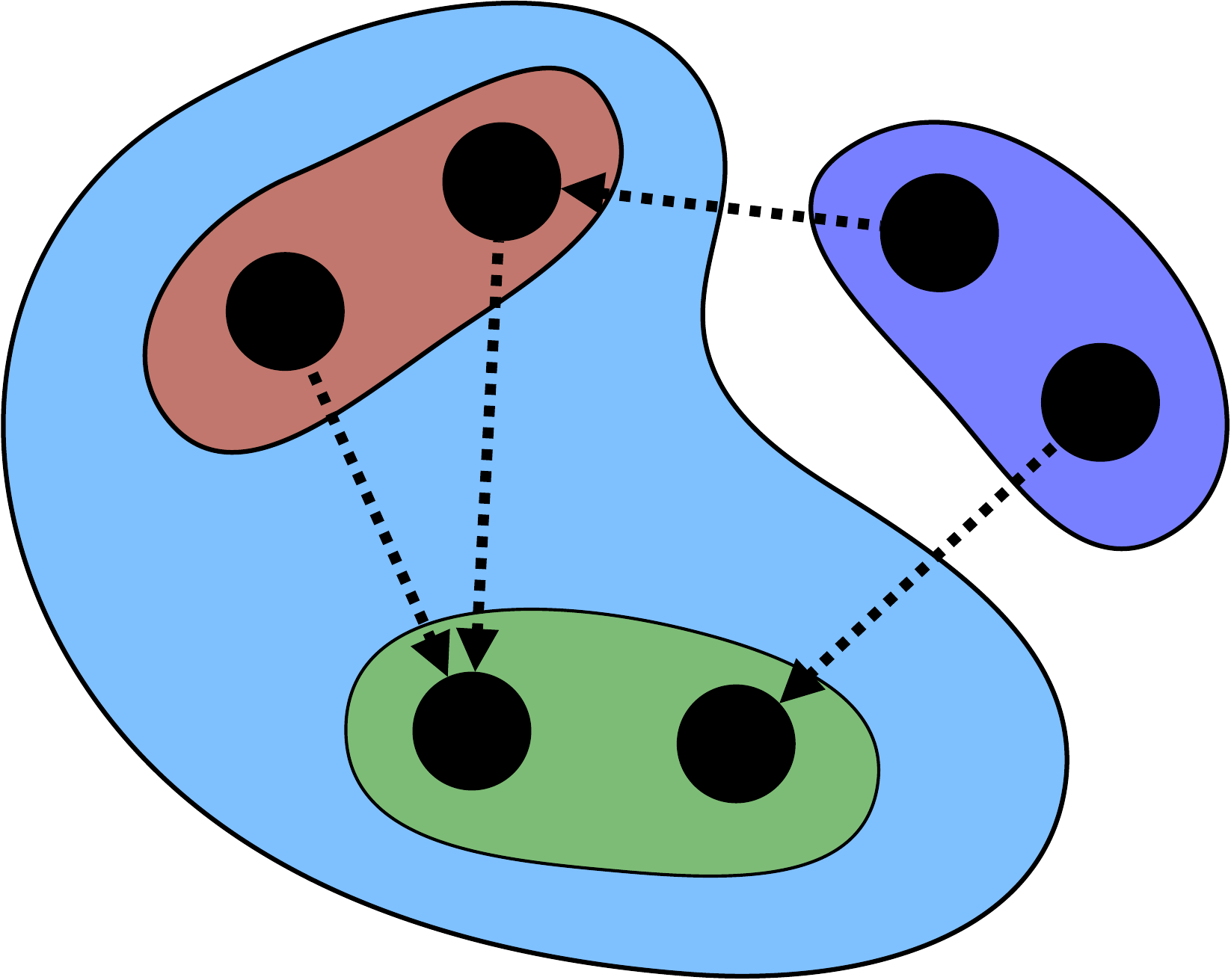}
    \caption{\textbf{Non-reducibility:} Consider six points equally spaced along a circle. They are partitioned into three clusters of points (red, green, and purple clusters in the figure). When the red and green clusters merge, they create a blue cluster. But $\Ch(\text{purple}, \text{blue}) < \min\{ \Ch(\text{purple}, \text{red}), \Ch(\text{purple}, \text{green})  \}$ \label{fig:non_reducible} since both vertices of the purple cluster have a neighbor (along the circle) in the blue cluster.}
\end{figure}

Note that the lack of reducibility prevents the nearest-neighbor chain algorithm~\cite{benzecri} from working since the new cluster to merge can be another cluster not contained in the stack of clusters traversed by the nearest-neighbor chain algorithm. We note that the nearest-neighbor chain technique is the main method for obtaining $O(n^2)$ time HAC implementations in the literature.

\section{Quadratic Time Algorithm}\label{sec:quadratic-time}

We now describe how to implement Chamfer-linkage HAC in $O(n^2)$ time using the \hacnn framework. We begin with a key property satisfied by Chamfer-linkage:

\begin{observation}\label{obs:min-monotone}
    Let $A,B$, and $C$ be disjoint clusters. Then,
    \begin{align*}
        \Ch(C, A \cup B) \le \min\{\Ch(C,A), \Ch(C,B)\}.
    \end{align*}
\end{observation}
This follows directly from the definition: for each point in $C$, the closest point in $A\cup B$ is at least as close as the closest point in $A$ or $B$ alone. See Figure~\ref{fig:Decreasing_Ch}.

\begin{figure}[t]
\vspace{-2em}
    \centering
    \begin{subfigure}[b]{0.19\textwidth}
        \centering
        \includegraphics[width=\textwidth, trim=0mm 30mm 0mm 0mm, clip]{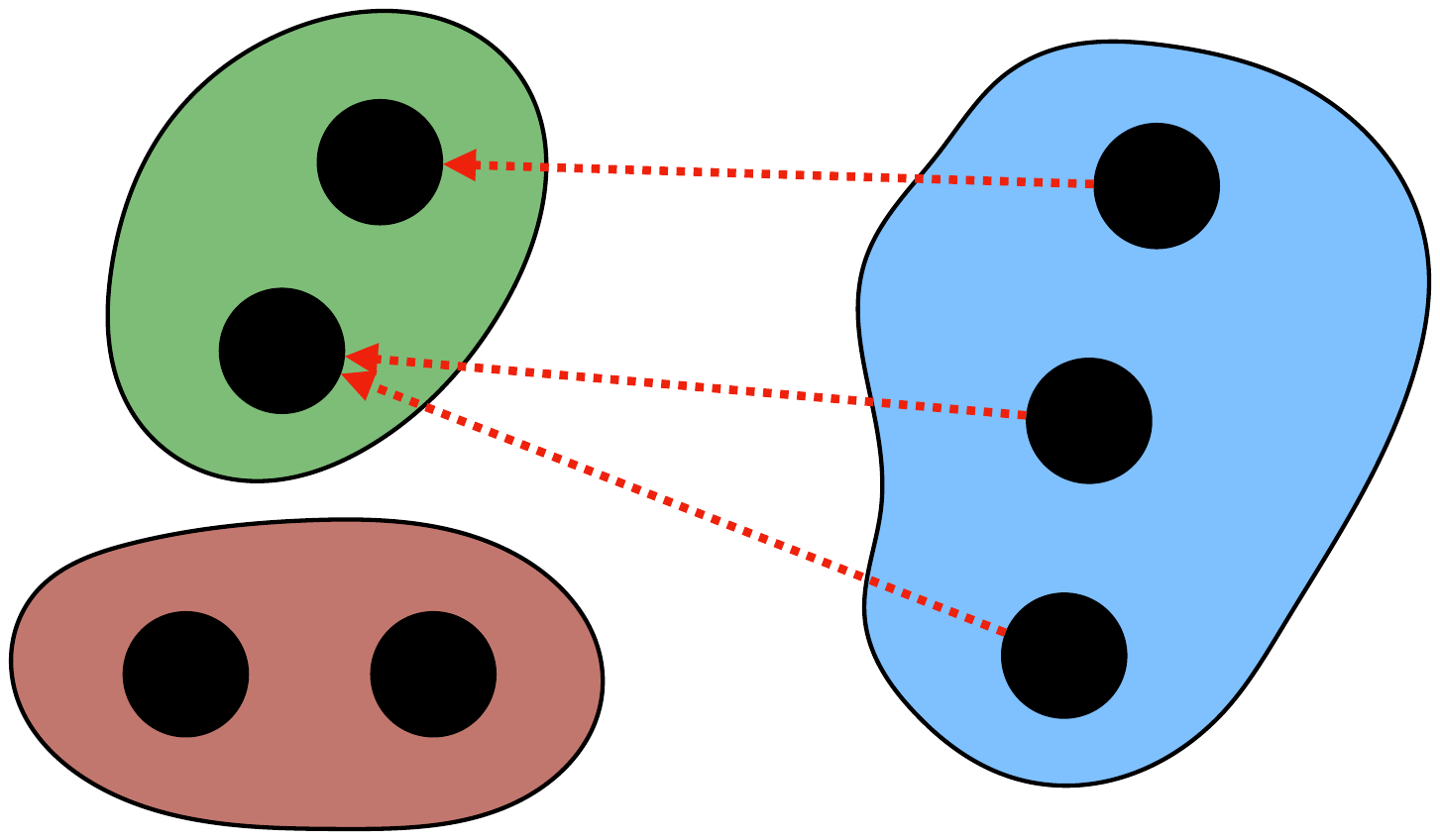}
        \caption{Before Merge.}\label{sfig:ch_non-monotone1}
    \end{subfigure}  \hspace{3.5em}
    \begin{subfigure}[b]{0.19\textwidth}
        \centering
        \includegraphics[width=\textwidth,trim=0mm 30mm 0mm 0mm, clip]{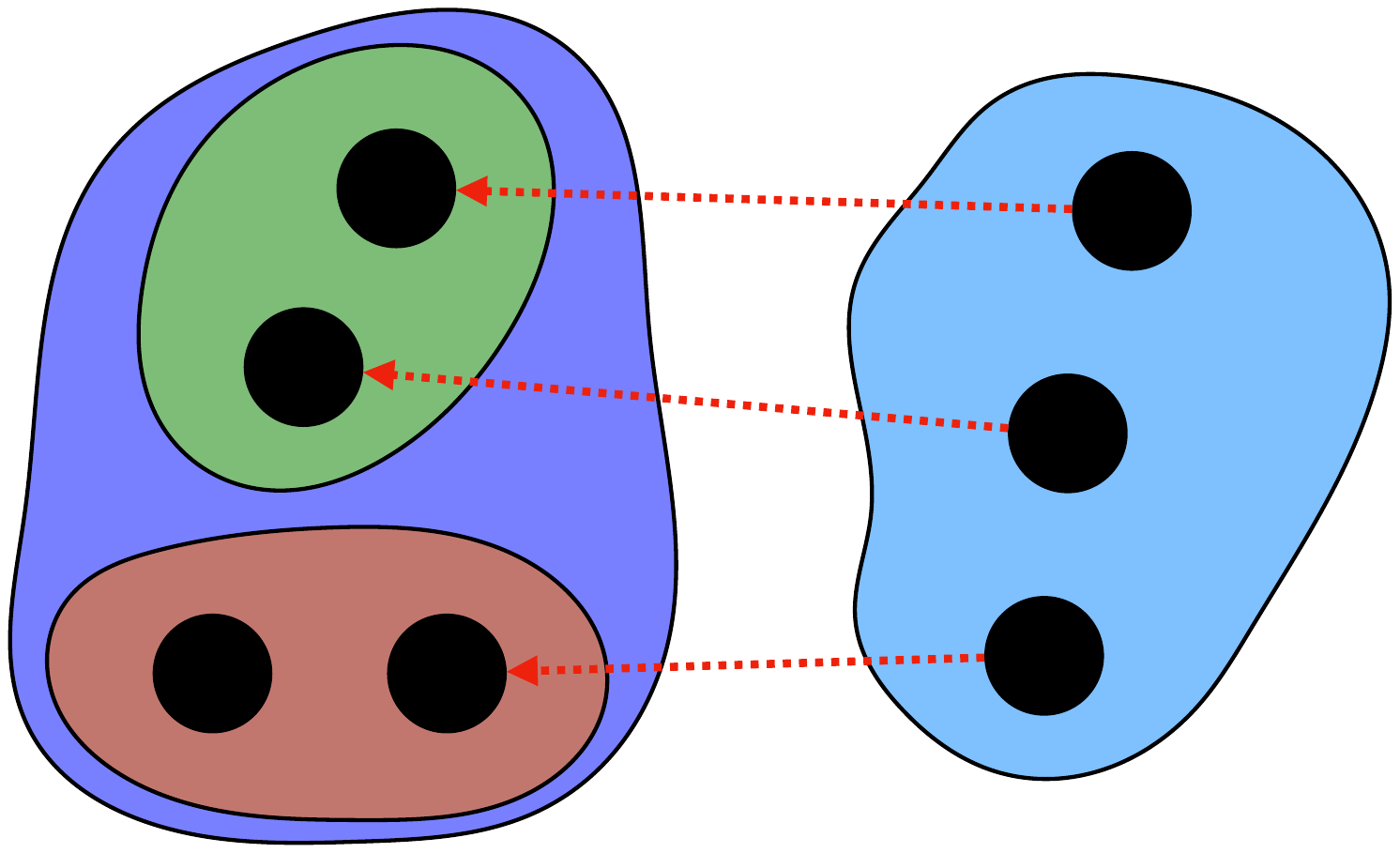}
        \caption{After Merge.}\label{sfig:ch_non-monotone2}
    \end{subfigure}   
    \caption{Demonstration of Observation~\ref{obs:min-monotone}} \label{fig:Decreasing_Ch}
\end{figure}

This observation implies that the Chamfer distance from a cluster $C$ to a newly formed cluster $A\cup B$ never increases. In particular, the nearest neighbor of $C$ either remains the same or becomes $A\cup B$. Furthermore, if either $A$ or $B$ was the nearest neighbor of $C$ before the merge, then $A\cup B$ is guaranteed to become its nearest neighbor afterward. Therefore, we can maintain the nearest neighbor of every active cluster by simply computing its distance to the newly merged cluster after each iteration.

Given this useful property, we now focus on efficiently computing the distances to and from the newly formed cluster $A\cup B$ after each merge.

For each active cluster $C\in \cC$, we maintain its Chamfer distance to every other active cluster, i.e., $\Ch(C,\cdot)$. However, this alone is not sufficient. The main challenge with Chamfer-linkage is that distances involving the newly formed cluster requires per-point nearest neighbor computations, which can be expensive. However, we show that these distances can be computed in $O(n)$ total time per merge by maintaining additional information: for each input point $p\in P$, we store its distance to every active cluster $C\in \cC$, defined as
\begin{align*}
d(p,C):= \min_{c\in C}d(p,c). 
\end{align*}
These values are updated after each merge.

Suppose two clusters $A$ and $B$ are merged to form $A\cup B$. Then for every point $p\in P$, we can compute its distance to $A\cup B$ via the identity
\begin{align}
    d(p,A\cup B)= \min\{d(p,A),d(p,B)\}.\label{eq:point_to_cluster}
\end{align}
Thus, all values $d(p,A\cup B)$ can be computed in $O(n)$ time using a single linear scan. Given these values, the incoming Chamfer distance from any active cluster $C$ to $A\cup B$ is given by
\begin{align}
    \Ch(C,A\cup B)=\sum_{p\in C} d(p,A\cup B).\label{eq:cluster_to_new_cluster}
\end{align}
Since each point belongs to exactly one active cluster, we can compute all incoming distances $\Ch(C,A\cup B)$ for all $C\in \cC$ in a single linear pass by aggregating the values computed in~\eqref{eq:point_to_cluster}.

The outgoing Chamfer distances are even simpler to compute. Since $A$ and $B$ are disjoint,
\begin{align}
    \Ch(A\cup B,C) = \Ch(A,C) +\Ch(B,C),\label{eq:new_cluster_to_cluster}
\end{align}
and thus all outgoing distances from $A\cup B$ can be computed in $O(n)$ time using the stored $\Ch(A,\cdot)$ and $\Ch(B,\cdot)$ values. The full implementation of the \merge subroutine for Chamfer-linkage is given in Algorithm~\ref{alg:Chamfer_HAC}.

\IncMargin{1.5em}
\begin{algorithm}[t]
\caption{Chamfer-linkage HAC}\label{alg:Chamfer_HAC}
Maintain $d(p,C)=\min_{c\in C}d(p,c)$, $\forall C\in \cC$\\
Maintain $\Ch(A,B)$, $\forall A,B\in \cC$\\
\BlankLine
\function{$\merge(A,B,\NN)$}{
    \tcp{Update distances from $A\cup B$}
    \For{$C \in \cC\setminus \{A,B\}$}{
        $\Ch(A\cup B, C) \gets \Ch(A,C) + \Ch(B,C)$\\
    }
    $\NN(A\cup B)\gets \argmin_{C\in\cC\setminus\{A,B\}}\Ch(A\cup B,C)$\\
    \BlankLine
    \tcp{Update distances to $A\cup B$}
    \For{$p \in P\setminus (A\cup B)$}{
        $d(p,A\cup B) \gets \min\{d(p,A),d(p,B)\}$
    }
    \For{$C\in \cC\setminus\{A,B\}$}{
        $\Ch(C,A\cup B) \gets \sum_{p\in C}d(p,A\cup B)$\\
        \uIf{$\NN(C)=A$ or $\NN(C)=B$}{
            $\NN(C)=A\cup B$\\
        }\ElseIf{$\Ch(C,A\cup B)<\Ch(C,\NN(C))$}{
            $\NN(C)=A\cup B$\\
        }
    }
}
\end{algorithm}
\DecMargin{1.5em}

\begin{theorem}\label{thm:Chamfer_quadratic}
Algorithm~\ref{alg:HAC_NN}, when instantiated with the Chamfer-linkage-specific \merge subroutine in Algorithm~\ref{alg:Chamfer_HAC}, correctly computes the Chamfer-linkage dendrogram in $O(n^2)$ time and space.
\end{theorem}
\begin{proof}
In each iteration, Algorithm~\ref{alg:HAC_NN} considers the nearest neighbor of every active cluster and merges the pair with minimum distance. Thus, correctness reduces to maintaining the correct Chamfer distances between cluster pairs. These distances are updated using Equations~\eqref{eq:cluster_to_new_cluster} and~\eqref{eq:new_cluster_to_cluster}, which correctly compute the necessary values after each merge (by definition of Chamfer distance). It follows that the algorithm produces the correct Chamfer-linkage dendrogram.

For the running time, each merge involves computing all outgoing distances from the new cluster $A\cup B$ and computing its nearest neighbor which takes $O(n)$ time in total. To compute all incoming distances to $A\cup B$, we first update the point-to-cluster distances in $O(n)$ time, then aggregate by cluster in another $O(n)$ time, and finally update the nearest neighbors of all affected clusters, which also takes $O(n)$ time. Thus, each iteration takes $O(n)$ time, resulting in a total of $O(n^2)$ time over all $n-1$ merges.

For space, we store Chamfer distances between all $O(n^2)$ (ordered) pairs of active clusters, and in addition maintain the distance from each input point to every active cluster. Both components require $O(n^2)$ space in total.
\end{proof}

\myparagraph{Variants of Chamfer-linkage}
The algorithm described above can be extended to handle the normalized ($\Ch_N$), symmetric ($\Ch_S$), and normalized symmetric ($\Ch_{NS}$) variants of Chamfer-linkage with minor modifications. 

We first consider the normalized variant $\Ch_N$. Since distances from a cluster $A$ to other clusters are all divided by the same normalizing term $|A|$, the cluster that minimizes the unnormalized Chamfer distance $\Ch(A, \cdot)$ also minimizes the normalized distance $\Ch_N(A, \cdot)$. Therefore, the nearest neighbor computed using $\Ch$ remains correct for $\Ch_N$, and the same $O(n^2)$-time algorithm applies. The final linkage values can be computed on the fly from the stored distances and cluster sizes.

In contrast, the symmetric variants $\Ch_S$ and $\Ch_{NS}$ do not satisfy Observation~\ref{obs:min-monotone}. In particular, the nearest neighbor of a cluster $C$ may change to another cluster after a merge if $A$ or $B$ were previously its nearest neighbor. Therefore, we can no longer maintain the nearest neighbor of each cluster using only the updates to $A\cup B$.

To address this, we explicitly maintain the final linkage values (e.g., $\Ch_S(A,B) = \Ch(A,B) + \Ch(B,A)$) in augmented maps, as described in Section~\ref{sec:prelims} for other linkage functions. In each merge step, the raw Chamfer distances $\Ch(\cdot, \cdot)$ are updated as before, and we use them to compute and update the final linkage values in the maps. Since these maps take $O(\log n)$ time per update, and we perform $O(n)$ updates per merge, the total time per iteration is $O(n \log n)$, leading to an overall time bound of $O(n^2 \log n)$ for the symmetric variants.

\begin{theorem}\label{thm:Chamfer_variants}
    There is an $O(n^2)$-time and $O(n^2)$-space algorithm for Chamfer-linkage HAC using $\Ch$ or $\Ch_N$, and an $O(n^2 \log n)$-time and $O(n^2)$-space algorithm for the symmetric variants $\Ch_S$ and $\Ch_{NS}$.
\end{theorem}

\section{Space--Time Trade-off Algorithm}\label{sec:time-space-trade-off}

We now present a space--time trade-off algorithm for Chamfer-linkage HAC. As discussed in Section~\ref{sec:quadratic-time}, Chamfer-linkage HAC can be implemented in $O(n^2)$ time and space. Here, we show that for any parameter $t\in [1,n]$, we can reduce the space usage to $O(n^2/t)$ by increasing the running time to $O(n^2t)$.

At a high-level, we maintain the same set of distances as in the quadratic-time algorithm, but restricted to the subset of \emph{large} clusters---i.e., those of size at least $t$. The remaining distances are computed on the fly when needed. Specifically, we maintain:
\begin{itemize}
    \item $d(p,C)$ for all $p\in P$ and all active clusters $C\in \cC$ such that $|C| \ge t$, and
    \item $\Ch(C,\cdot)$ for all active clusters $C\in \cC$ such that $|C|\ge t$.
\end{itemize}
Since the number of large clusters is at most $n/t$ at any point, the overall space usage is $O(n^2/t)$.

To compute the outgoing Chamfer distances from $A\cup B$ to other clusters, we apply Equation~\eqref{eq:new_cluster_to_cluster} as in the quadratic-time algorithm. If both $A$ and $B$ are large, then the Chamfer distances $\Ch(A,\cdot)$ and $\Ch(B,\cdot)$ are already stored, and we can directly compute $\Ch(A\cup B,\cdot)$. Otherwise, suppose $A$ is a small cluster. In this case, we first compute the Chamfer distance from $A$ to every other active cluster. This involves computing, for each $a\in A$, the closest point to each cluster $C\in \cC$. Since $|A|<t$, this requires $O(nt)$ time. Once the point-to-cluster distances are computed, we aggregate them to obtain $\Ch(A,\cdot)$. We then apply Equation~\ref{eq:new_cluster_to_cluster} to compute the distances from $A\cup B$ to other clusters. After computing these values, we determine the nearest neighbor of $A\cup B$ and store it. If $A\cup B$ is large, we additionally store $\Ch(A\cup B,\cdot)$.

For incoming Chamfer distances to $A\cup B$, we must compute $d(p,A)$ and $d(p,B)$ for each point $p\in P\setminus(A\cup B)$. If either $A$ or $B$ are large, the corresponding distances are already stored; otherwise, we compute them from scratch. Since small clusters have at most $t$ points, this step requires $O(nt)$ time. We then apply Equation~\eqref{eq:point_to_cluster} to compute $d(p,A\cup B)$ from $d(p,A)$ and $d(p,B)$, and store the result if $A\cup B$ is large. Aggregating these point-level distances by cluster yields the incoming Chamfer distance $\Ch(C,A\cup B)$ for all $C\in \cC$. We then update the stored distances for large clusters and update the nearest neighbor for each cluster as needed.

Thus, the time spent in each iteration is $O(nt)$, resulting in the overall time bound of $O(n^2t)$. Note that this algorithm works only for the variants of Chamfer-linkage using $\Ch$ and $\Ch_N$; as discussed in Section~\ref{sec:quadratic-time}, the symmetric variants do not satisfy the key property in Observation~\ref{obs:min-monotone}, and may require recomputing the nearest neighbor of every other cluster from scratch, which is highly inefficient.

\begin{theorem}\label{thm:space_time_tradeoff}
    For any parameter $t \in [1, n]$, there is an algorithm for Chamfer-linkage HAC using $\Ch$ or $\Ch_N$ that runs in $O(n^2t)$ time and uses $O(n^2/t)$ space.
\end{theorem}
\section{Empirical Evaluation}\label{sec:experiments}
We experimentally evaluate our implementation of Chamfer-linkage HAC against five classical linkage functions---average, centroid, complete, single, and Ward's---across $19$ benchmark clustering datasets, using four standard clustering metrics (ARI, NMI, AMI, and FMI), and obtain the following main results:
\begin{enumerate}[topsep=0pt,itemsep=0pt,parsep=0pt,leftmargin=15pt]
\item Chamfer-linkage with the asymmetric, non-normalized variant ($\Ch$) consistently outperforms all classical linkage functions in clustering quality---achieving up to \textbf{57\%} improvement in ARI over the best baseline, \textbf{6\%} improvement on average, and never more than \textbf{8\%} worse.
\item In contrast, even the strongest classical baseline, Ward's-linkage, is not consistently reliable: it can be up to \textbf{29\%} worse than the best-performing method in terms of ARI in the worst case, and is \textbf{4\%} worse on average.
\item Chamfer-linkage (\Ch) produces balanced dendrograms with low height, comparable to those produced by average- and Ward's-linkage.
\item We provide fast C++ implementations of Chamfer-linkage, and optimized implementations of the five classical linkages. In particular, Chamfer-linkage incurs no slowdown in running time compared to the classical baselines in our experiments, while our optimized implementations achieve up to \textbf{5.75--9.28\boldmath$\times$} speedup over highly optimized libraries such as \texttt{fastcluster} and {\tt scikit-learn}, respectively.
\end{enumerate}

\begin{table}[t]
\centering
\small
\begin{tabular}{lcccl}
\toprule
Dataset & n & d & k & Type \\
\midrule
{\tt iris}   & 150   & 4    & 3  & Tabular \\
{\tt wine}   & 178   & 13   & 3  & Tabular \\
{\tt cancer} & 569   & 30   & 2  & Tabular \\
{\tt digits} & 1797  & 64   & 10 & Raw Images \\
{\tt faces}  & 400   & 4096 & 40 & Raw Images \\
{\tt coil-20}  & 1440   & 16384 & 20 & Raw Images \\
{\tt coil-100}  & 7200   & 49152 & 100 & Raw Images \\
{\tt alloprof}  & 2556   & 4096 & 13 & Text Embeddings \\
{\tt usps}  & 9298   & 256 & 10 & Raw Images \\
{\tt hal}  & 26233   & 4096 & 10 & Text Embeddings \\
{\tt news}  & 59545 & 4096  & 20 & Text Embeddings \\
{\tt cifar-10}  & 60000 & 512  & 10 & Image Embeddings \\
{\tt cifar-100}  & 60000 & 512  & 100 & Image Embeddings \\
{\tt mnist}  & 70000 & 784  & 10 & Raw Images \\
{\tt fashion-mnist}  & 70000 & 512  & 10 & Image Embeddings \\
{\tt birds}  & 84635 & 1024 & 525 & Image Embeddings\\ 
{\tt food-101}  & 101000 & 512 & 101 & Image Embeddings\\ 
{\tt emnist}  & 131600 & 512 & 47 & Image Embeddings\\ 
{\tt reddit}  & 420464 & 1024	& 50 & Text Embeddings \\ 
{\tt covertype}  & 581012 & 54	& 7 & Tabular\\ 
\bottomrule
\end{tabular}
\caption{\small Datasets used in this paper.}
\label{tab:datasets}
\vspace{-0.5cm}
\end{table}


\begin{table*}[ht]
\centering
\vspace{-1em}
\caption{\small The Adjusted Rand Index (ARI) and Normalized Mutual Information (NMI) scores of the four Chamfer-linkage variants ($\Ch$, $\Ch_N$, $\Ch_S$ and $\Ch_{NS}$) versus baseline classical linkage functions. 
The best score for each dataset is shown in \textcolor{Green}{\underline{\textbf{green bold}}}. Additionally, whenever a Chamfer variant outperforms all baselines on a dataset, it is shown in \textcolor{Blue}{\textbf{blue bold}}.
}
\label{tab:quality-evals}
\small
\begin{tabular}{cc|ccccc|cccc}
\toprule
\multicolumn{1}{l}{} &
  Dataset &
  Average &
  Centroid &
  Complete &
  Single &
  Ward's &
  $\Ch$ &
  $\Ch_N$ &
  $\Ch_S$ &
  $\Ch_{NS}$\\
\midrule
\multirow{18}{*}{\begin{sideways}ARI\end{sideways}} 
     & {\tt iris}       & 0.611 & 0.759 & 0.642 & 0.738 & 0.731 & \better{0.759} & \best{0.775} & 0.718 & 0.749\\
     & {\tt wine}       & 0.352 & 0.352 & 0.371 & 0.3 & 0.368 & 0.359 & 0.359 & \better{0.391} & \best{0.402}\\
     & {\tt cancer}     & 0.537 & 0.509 & 0.465 & \best{0.574} & 0.406 & 0.541 & 0.438 & 0.421 & 0.476\\
     & {\tt digits}     & 0.799 & 0.718 & 0.519 & 0.784 & 0.845 & \best{0.875} & \better{0.874} & 0.762 & 0.792 \\
     & {\tt faces}      & 0.6 & 0.539 & 0.496 & 0.626 & 0.613 & \better{0.647} & 0.624 & 0.625 & \best{0.652} \\
     & {\tt coil-20}    & 0.59 & 0.632 & 0.555 & \best{0.849} & 0.701 & 0.777 & 0.67 & 0.762 & 0.651 \\
     & {\tt coil-100}   & 0.637 & 0.647 & 0.589 & \best{0.798} & 0.673 & 0.742 & 0.711 & 0.754 & 0.685 \\
     & {\tt alloprof}   & \best{0.801} & 0.777 & 0.734 & 0.726 & 0.783 & 0.735 & 0.785 & 0.785 & 0.739 \\
     & {\tt usps}       & 0.573 & 0.499 & 0.377 & 0.373 & 0.676 & \best{0.798} & \better{0.720} & \better{0.693} & 0.576 \\
     & {\tt hal}        & 0.293 & 0.304 & 0.248 & 0.044 & 0.352 & \best{0.408} & 0.211 & 0.311 & 0.317 \\
     & {\tt news}       & 0.467 & 0.398 & 0.414 & 0.284 & 0.464 & \best{0.470} & 0.459 & 0.456 & 0.449 \\
     & {\tt cifar-10}   & 0.7 & 0.339 & 0.379 & 0.124 & 0.69 & 0.689 & \best{0.71} & 0.534 & 0.580 \\
     & {\tt cifar-100}  & 0.308 & 0.155 & 0.213 & 0.095 & 0.303 & \best{0.318} & \better{0.31} & 0.262 & 0.275 \\
     & {\tt mnist}      & 0.489 & 0.328 & 0.251 & 0.22 & 0.595 & \best{0.816} & \better{0.782} & \better{0.611} & 0.53 \\
     & {\tt fashion-mnist} & 0.526 & 0.450 & 0.302 & 0.248 & 0.557 & \best{0.610} & \better{0.574} & 0.475 & \better{0.563} \\
     & {\tt birds}      & 0.649 & 0.587 & 0.587 & 0.491 & \best{0.752} & 0.746 & 0.633 & 0.691 & 0.612 \\
     & {\tt food-101}   & 0.467 & 0.166 & 0.260 & 0.127 & 0.449 & \best{0.478} & 0.455 & 0.325 & 0.376 \\
     & {\tt emnist}     & 0.258 & 0.201 & 0.121 & 0.072 & 0.263 & \best{0.412} & \better{0.392} & \better{0.278} & \better{0.319} \\
\cmidrule{2-11}
     & {\bf Geomean}    & 0.511 & 0.418 & 0.382 & 0.301 & 0.538 & \best{0.595} & \better{0.547}  & 0.502 & 0.530\\
\midrule 
\multirow{18}{*}{\begin{sideways}NMI\end{sideways}} 
     & {\tt iris}       & 0.734 & \best{0.806} & 0.722 & 0.734 & 0.77 & \best{0.806} & \best{0.806} & 0.763 & 0.784 \\
     & {\tt wine}       & 0.428 & 0.428 & \best{0.442} & 0.417 & 0.428 & 0.428 & 0.428 & 0.431 & 0.417 \\
     & {\tt cancer}     & 0.457 & 0.428 & 0.43 & 0.371 & 0.423 & \best{0.465} & 0.394 & 0.425 & 0.439 \\
     & {\tt digits}     & 0.852 & 0.805 & 0.72 & 0.817 & 0.881 & \better{0.904} & \best{0.908} & 0.828 & 0.849 \\
     & {\tt faces}      & 0.866 & 0.862 & 0.851 & \best{0.884} & 0.869 & 0.882 & 0.88  & 0.869 & 0.883 \\
     & {\tt coil-20}    & 0.806 & 0.821 & 0.786 & \best{0.928} & 0.84 & 0.898 & 0.863 & 0.875 & 0.84 \\
     & {\tt coil-100}   & 0.879 & 0.875 & 0.85 & \best{0.938} & 0.881 & 0.917 & 0.913 & 0.918 & 0.903 \\
     & {\tt alloprof}   & \best{0.803} & 0.781 & 0.788 & 0.737 & 0.779 & 0.778 & 0.799 & 0.797 & 0.743 \\
     & {\tt usps}       & 0.684 & 0.612 & 0.55 & 0.512 & 0.76 & \best{0.823} & \better{0.761} & 0.743 & 0.685 \\
     & {\tt hal}        & 0.394 & 0.391 & 0.378 & 0.336 & 0.399 & \best{0.415} & 0.387 & 0.374 & 0.380 \\
     & {\tt news}       & 0.623 & 0.618 & 0.592 & 0.596 & 0.625 & 0.617 & \best{0.641} & 0.618 & 0.597 \\
     & {\tt cifar-10}   & 0.712 & 0.499 & 0.572 & 0.387 & 0.746 & \best{0.767} & 0.718 & 0.588 & 0.691 \\
     & {\tt cifar-100}  & 0.602 & 0.603 & 0.596 & 0.615 & 0.595 & 0.597 & \best{0.649} & 0.610 & 0.593 \\
     & {\tt mnist}      & 0.63 & 0.522 & 0.484 & 0.406 & 0.714 & \best{0.853} & \better{0.796} & 0.71 & 0.645 \\
     & {\tt fashion-mnist} & 0.635 & 0.566 & 0.503 & 0.414 & 0.689 & \best{0.716} & 0.651 & 0.578 & 0.683 \\
     & {\tt birds}      & 0.905 & 0.872 & 0.882 & 0.833 & 0.929 & \best{0.935} & 0.9 & 0.909 & 0.89 \\
     & {\tt food-101}   & 0.671 & 0.613 & 0.605 & 0.602 & 0.659 & \better{0.687} & \best{0.693} & 0.635 & 0.607 \\
     & {\tt emnist}     & 0.545 & 0.509 & 0.503 & 0.516 & 0.543 & \best{0.644} & \better{0.620} & \better{0.552} & \better{0.552} \\
\cmidrule{2-11}
     & {\bf Geomean}        & 0.661 & 0.624 & 0.606 & 0.580 & 0.675 & \best{0.708} & \better{0.688} & 0.650 & 0.663\\
\bottomrule
\end{tabular}
\end{table*}
\myparagraph{Experimental Setup} We run our experiments on a 96-core Dell PowerEdge R940 machine (with two-way hyperthreading), equipped with 4$\times$2.4GHz Intel 24-core 8160 Xeon processors (with 33MB L3 cache) and 1.5TB of main memory. All implementations are written in C++ and compiled using \texttt{g++} (version 11.4) with the \texttt{-O3} optimization flag. We also provide Python bindings for our implementations.

\myparagraph{Datasets}
We use a variety of real-world clustering datasets, summarized in Table~\ref{tab:datasets} and described in detail in Appendix~\ref{app:datasets}.

\myparagraph{Algorithms Implemented}
We implement all four Chamfer-linkage variants---$\Ch$, $\Ch_N$, $\Ch_S$, and $\Ch_{NS}$---in the \hacnn framework described in Section~\ref{sec:quadratic-time}. We also implement five widely used classical linkage functions---average, centroid, complete, single, and Ward's---as baselines, all within the same \hacnn framework.

Our implementations use the ParlayLib library~\cite{blelloch2020parlaylib} to leverage shared-memory multicore parallelism. 
Compared to highly optimized implementations from existing libraries such as \texttt{fastcluster} and \texttt{scikit-learn}, our versions provide significant speedups: by up to $5.75$--$9.28\times$ respectively; see Appendix~\ref{app:running_times} for more details.

Finally, we implement the space--time trade-off algorithm for Chamfer-linkage from Section~\ref{sec:time-space-trade-off}, and evaluate it on the larger datasets where the other $O(n^2)$-space methods fail to run.

As a baseline for the large-scale setting, we also evaluate the recent $(1+\eps)$-approximate centroid-linkage algorithm from~\cite{Bateni2025EfficientCentroid}, which achieves subquadratic running time and near-linear space. We use the recommended setting $\eps = 0.1$, and refer to this algorithm as $\text{Centroid}_{0.1}$.

\begin{table*}[t]
\centering
\vspace{-1em}
\caption{\small Dendrogram heights for Chamfer-linkage and the various baselines across all datasets. For each dataset, the \emph{Geomean} column reports the geometric mean of the dendrogram heights across all methods. The final \emph{Score} row shows the average ratio of each method's dendrogram height to the geometric mean for that dataset (lower means more balanced).
}
\label{tab:dendrogram_height}
\small
\begin{tabular}{cc|cccccc|cccc|c}
\toprule
  Dataset &
  n &
  Average &
  Centroid &
  $\text{Centroid}_{0.1}$ &
  Complete &
  Single &
  Ward's &
  $\Ch$ &
  $\Ch_N$ &
  $\Ch_S$ &
  $\Ch_{NS}$ &
  Geomean\\
\midrule
     {\tt iris} & 150        & 13 & 15 & 14 & 11 & 31 & 11 & 12 & 26 & 11 & 16 & 14.96 \\
     {\tt wine} & 178      & 12 & 14 & 14 & 11 & 25 & 11 & 14 & 15 & 11 & 13 & 13.58 \\
     {\tt cancer} & 569    & 18 & 21 & 20 & 16 & 131 & 14 & 20 & 31 & 14 & 22 & 23.05\\
     {\tt digits} & 1797    & 24 & 102 & 114 & 18 & 257 & 16 & 44 & 65 & 13 & 61 & 46.46\\
     {\tt faces}  & 400    & 25 & 83 & 78 & 15 & 85 & 14 & 24 & 76 & 11 & 44 & 34.73 \\
     {\tt coil-20} & 1440   & 27 & 44 & 45 & 22 & 98 & 15 & 27 & 53 & 14 & 35 & 32.24\\
     {\tt coil-100} & 7200  & 38 & 112 & 121 & 25 & 263 & 20 & 37 & 175 & 16 & 78 & 59.4\\
     {\tt alloprof} & 2556 & 40  & 108 & 108 & 23 & 355 & 20 & 33 & 124 & 15 & 80 & 57.19\\
     {\tt usps} & 9298     & 75 & 704 & 754 & 36 & 4797 & 21 & 61 & 882 & 16 & 321 & 182.01 \\
     {\tt hal} & 26233     & 65 & 725 & 920 & 32 & 12945 & 24 & 68 & 1584 & 19 & 330 & 222.03 \\
     {\tt news} & 59545    & 271 & 1116 & 3171 & 263 & 2585 & 30 & 260 & 1018 & 255 & 453 & 476.07 \\
     {\tt cifar-10} & 60000   & 182 & 7189 & 8932 & 33 & 37831 & 25 & 497 & 8368 & 18 & 1402 & 721.60 \\
     {\tt cifar-100} & 60000  & 157 & 9990 & 11590 & 35 & 38960 & 25 & 644 & 14766 & 19 & 1665 & 845.37 \\
     {\tt mnist} & 70000   & 95 & 3320 & 4157 & 35 & 40931 & 25 & 113 & 4473 & 21 & 856 & 460.26 \\
     {\tt fashion-mnist} & 70000   & 101 & 1916 & 3064 & 33 & 32681 & 25 & 227 & 12032 & 19 & 677 & 473.09 \\
     {\tt birds} & 84635   & 126 & 2301 & 3088 & 40 & 16757 & 36 & 105 & 1575 & 21 & 739 & 375.31 \\
     {\tt food-101} & 101000   & 189 & 19120 & 20818 & 37 & 68777 & 27 & 292 & 8947 & 19 & 2560 & 958.84 \\
     {\tt emnist} & 131600   & 159 & 3630 & 5818 & 40 & 91349 & 30 & 286 & 20385 & 20 & 1432 & 756.59 \\
\midrule
     \multicolumn{2}{c|}{\bf Score}  & 0.490 & 4.781 & 5.946 & 0.327 & 34.165 & 0.260 & 0.595 & 7.408 & 0.259 & 1.516 & 1 \\
\bottomrule
\end{tabular}
\end{table*}

\begin{table}[t]
\centering
\vspace{-1em}
\caption{\small The Adjusted Rand Index (ARI) and Normalized Mutual Information (NMI) scores of Chamfer- and approximate centroid-linkage on {\tt reddit} and {\tt covertype}. The best score for each dataset is shown in \textcolor{Green}{\underline{\textbf{green bold}}}.
}
\label{tab:large-evals}
\small
\begin{tabular}{cc|c|c}
\toprule
\multicolumn{1}{l}{} &
  Dataset &
  $\text{Centroid}_{0.1}$ &
  $\Ch$ \\
\midrule
\multirow{2}{*}{\begin{sideways}ARI\end{sideways}} 
     & {\tt reddit}       & 0.064 & \best{0.38}\\
     & {\tt covertype}       & \best{0.547} & \best{0.547}\\
\midrule 
\multirow{2}{*}{\begin{sideways}NMI\end{sideways}} 
     & {\tt reddit}       & \best{0.538} & 0.525 \\
     & {\tt covertype}       & 0.181 & \best{0.188}\\
\bottomrule
\end{tabular}
\vspace{-1.5em}
\end{table}
\subsection{Quality Evaluation}\label{sec:qual_evals}
We evaluate clustering quality of our algorithms against ground truth labelings using standard metrics such as \emph{Adjusted Rand Index (ARI)}, \emph{Normalized Mutual Information (NMI)}, \emph{Adjusted Mutual Information (AMI)} and \emph{Fowlkes-Mallows Index (FMI)}. 
Given a dendrogram, these scores are computed by considering the best score obtained over the clusterings formed by all possible \emph{thresholded-cuts} to the dendrograms; see Appendix~\ref{app:quality_evals} for a more detailed description.

\myparagraph{Results}
Table~\ref{tab:quality-evals} contains the results of our quality evaluations with respect to ARI and NMI scores; AMI and FMI scores are presented in Table~\ref{tab:quality-evals_other} in the Appendix.

We observe that Chamfer-linkage using the asymmetric, non-normalized variant ($\Ch$), consistently performs well across all datasets. It is always within \textbf{8\%} of the best-performing baseline in terms of ARI (worst case), and outperforms the best baseline by up to \textbf{57\%} in the best case (on \texttt{emnist}), with an average improvement of \textbf{6\%}. For NMI, $\Ch$ is always within \textbf{3\%} of the best baseline in the worst case, and improves upon it by up to \textbf{20\%} in the best case (on \texttt{mnist}), with an average gain of \textbf{2\%}. Notably, the best classical baseline---average and Ward's---can be up to \textbf{29\%} and \textbf{9\%} worse than the best-performing baseline in terms of ARI and NMI, respectively, underscoring the consistency and robustness of Chamfer-linkage. We see similar performance gains in AMI and FMI as well (see Appendix~\ref{app:quality_evals}).

Among the other Chamfer variants, $\Ch_N$ can achieve sizable best-case gains over the best baseline---up to \textbf{49\%} in ARI and \textbf{14\%} in NMI (both on \texttt{emnist}). On average, however, none of the variants improves over the best baseline, and they can be substantially worse on some datasets---up to \textbf{40\%} (ARI) and \textbf{21\%} (NMI) worse. 
Nonetheless, all three variants---$\Ch_N$, $\Ch_S$ and $\Ch_{NS}$---remain competitive: their average gaps to the best method are \textbf{1\%}, \textbf{6\%}, and \textbf{10\%} in ARI, and \textbf{2\%}, \textbf{4\%}, and \textbf{5\%} in NMI; by comparison, average- and Ward's-linkage have average gaps of \textbf{9\%} and \textbf{4\%} in ARI, and \textbf{6\%} and \textbf{3\%} in NMI.

Single-, complete- and centroid-linkage demonstrate substantial variability in performance. For instance, single-linkage is the best-performing method (including Chamfer variants) on certain datasets (like {\tt cancer}, {\tt coil-20}, etc.), but can perform up to \textbf{88\%} worse than the best-performing baseline (in \texttt{hal}), and is \textbf{37\%} worse on average in terms of ARI. 
Similarly, complete- and centroid-linkage can be up to \textbf{58\%} and \textbf{65\%} worse than the best-performing baseline, and \textbf{30\%} and \textbf{23\%} worse on average, respectively, in terms of ARI. We observe similar trends for NMI, AMI and FMI as well.

\myparagraph{Discussion}
We briefly discuss why certain Chamfer variants perform better than others in our experiments.
We observed that asymmetric variants generally seem to outperform the symmetric ones; indeed they also outperform all baselines on average. 
This is possibly because $\Ch$ and $\Ch_N$ prioritize directional proximity, allowing the algorithm to merge a small, coherent cluster into a larger one without requiring mutual closeness. 
In contrast, symmetric variants like $\Ch_S$ tend to favor merges between clusters that are mutually close in both directions. 
This can encourage merges between clusters containing shared concepts, which may be beneficial on some datasets (e.g., overlapping/correlated classes), but overly conservative on others.

Comparing the unnormalized and normalized variants, we observe that $\Ch$ often performs better than $\Ch_N$. 
A possible reason is that $\Ch$ gives precedence to merging small clusters early. 
This property helps form tight subclusters that then later integrate well. 
In contrast, $\Ch_N$ merges are driven by the average distance, which may cause large clusters to merge earlier than desirable in some cases, leading to overly skewed clusterings. This can be observed in Section~\ref{sec:balanced_hierarchies}, where we observe that $\Ch_N$ tends to produce relatively taller hierarchies.

Overall, Chamfer-linkage with $\Ch$ strikes a good balance across the various design criteria and desiderata outlined in Section~\ref{sec:intro}, and its strong empirical performance across the diverse set of datasets we consider further supports it as a robust practical choice.

\subsection{Scaling to large datasets}
None of the exact classical HAC baselines or Chamfer variants were able to scale to the largest datasets in our benchmark ({\tt reddit} and {\tt covertype}) due to their $O(n^2)$ space usage. All of them ran out of memory on our machine with 1.5TB of RAM. The only baseline that completed on these datasets was the approximate centroid-linkage algorithm ($\text{Centroid}_{0.1}$). 

To evaluate on this scale, we run our space--time trade-off algorithm for Chamfer-linkage with the setting $t=10$. As shown in Table~\ref{tab:large-evals}, it significantly outperforms $\text{Centroid}_{0.1}$ in ARI on reddit, and matches it in every other aspect. We note that the running times for Chamfer-linkage are significantly higher due to the superquadratic cost of the algorithm (see Appendix~\ref{app:running_times}), but the ability to complete and consistently perform well, even at this scale, highlights the potential of Chamfer-linkage and motivates the development of more scalable and parallel implementations.

\subsection{Dendrogram Height} \label{sec:balanced_hierarchies}
We evaluate the height of the dendrograms produced by Chamfer-linkage and the baseline linkage functions as a proxy for their structural balance. To compare across datasets of different scales, we normalize by the geometric mean of the dendrogram heights across all methods on each dataset. We define a \emph{score} for each method as the average ratio between its dendrogram height and the corresponding geometric mean height---lower values indicate more balanced dendrograms, with values below~$1$ representing better-than-average performance. See Table~\ref{tab:dendrogram_height} for the raw heights and these scores.

\myparagraph{Discussion} We observe that classical linkage methods such as complete, average and Ward's consistently produce low-height dendrograms, while single-linkage yields the tallest hierarchies. 

Among the Chamfer variants, both $\Ch$ and $\Ch_S$ yield relatively low-height dendrograms. In particular, $\Ch_S$ consistently produces the most balanced hierarchies across datasets. This may be attributed to the symmetry in $\Ch_S$, which encourages merges between similarly-sized clusters that are mutually close or equally relevant. Such a property naturally leads to more balanced trees. In contrast, $\Ch_N$ often produces significantly taller hierarchies, supporting our earlier hypothesis that normalization may cause large clusters to merge prematurely, resulting in chaining and skewed hierarchies.
\section{Related Work on HAC}
Hierarchical Agglomerative Clustering (HAC) has a long history of study going back over a century.
Many of the widely used linkage functions today (e.g., single-, complete-, average-, centroid-, Ward's, etc.) are captured by the Lance--Williams update framework~\cite{lance1967general}. 
A seminal and early line of work starting in the 1970s studied how to compute these classical linkages efficiently and exactly, including $O(n^2)$-time algorithms in settings where an explicit distance matrix is available (e.g., via nearest-neighbor chain methods and related techniques)~\cite{benzecri,irbook}. 
Moreover, for exact HAC, truly subquadratic running times seem unlikely in general, since even identifying the first merge requires solving a closest-pair-type problem: a canonical problem in fine-grained complexity theory.
On the systems side, modern libraries such as {\tt fastcluster} implement highly optimized sequential routines for these classical linkages that run in $O(n^2)$ time and are widely used in practice~\cite{mullner2013fastcluster}.

More recently, a line of work has studied how to scale hierarchical clustering to large datasets, often by relaxing or modifying parts of the standard HAC pipeline: e.g., graph- and MST-based formulations~\cite{affinity,monath2021scalable,dhulipala2021hierarchical}, parallel and distributed computation~\cite{Dhulipala2024OptimalParallel,yu2021parchain, monath2021scalable}, or approximate merge rules~\cite{dhulipala2022hierarchical,dhulipala2023terahac, cochez2015twister,Bateni2025EfficientCentroid,abboudhac}. 
These methods primarily focus on scalability (e.g., parallel or distributed settings, and studying billion-scale or larger clustering instances) while typically benchmarking clustering quality against classical HAC linkages, which remain the standard reference points in both software toolkits and empirical evaluations.

In this work, we focus on a more basic and arguably fundamental question: \emph{what is the right linkage for HAC?} 
Improved results for this question influence both lines of work mentioned above: improved linkage functions should make their way into popular libraries such as {\tt SciPy} and {\tt fastcluster}, as well as expand the set of linkage functions studied by the scalability-oriented literature.
An immediate next step building on this work will be to study scalable and approximate variants of Chamfer-linkage.

\section{Conclusion}
In this work, we introduced Chamfer-linkage, a novel linkage function for Hierarchical Agglomerative Clustering that cleanly captures concepts in one cluster being well represented in another cluster. We showed that despite its complexity, Chamfer-linkage HAC can be computed in $O(n^2)$ time, matching the efficiency of classical linkage methods. Furthermore, we provided a practical space--time trade-off algorithm which extends its applicability to larger clustering inputs. Our empirical evaluation shows that Chamfer-linkage consistently outperforms widely used classical linkage functions across a diverse collection of real-world datasets, achieving significant gains in clustering quality. 
For future work, we are interested in understanding how to parallelize Chamfer-linkage HAC, and to develop rigorous parallel approximation algorithms that can scale to very large datasets in shared-memory and distributed environments without sacrificing quality.

\begin{acks}
This work is supported by NSF grants CCF-2403235, CNS-2317194 and CCF-2403236. 
\end{acks}

\bibliographystyle{ACM-Reference-Format}
\bibliography{references}

\appendix

\section{Empirical Evaluation continued}\label{app:empirical_evals}
\subsection{Datasets}\label{app:datasets}
\myparagraph{Datasets}
We use a variety of real-world clustering datasets, summarized in Table~\ref{tab:datasets} and described below:
\begin{itemize}
\item {\tt iris, wine, cancer,  digits, faces} and {\tt covertype} are standard clustering datasets from the UCI repository~\cite{UCI} (CC BY 4.0).
\item {\tt mnist}\cite{mnist} (CC BY-SA 3.0 DEED) is a standard machine learning dataset that consists of $28\times 28$ dimensional grayscale images of handwritten digits ($0$--$9$). Each digit represents a distinct cluster.
\item {\tt usps}~\cite{hull2002database} (CC BY 4.0) is another classic dataset consisting of grayscale images of handwritten digits, but at a lower resolution ($16\times 16$).
\item {\tt coil-20}~\cite{datasetcoil20} and {\tt coil-100}~\cite{datasetcoil100} (Non-Commercial Research Purposes Only) are object image datasets with grayscale images taken from varying angles. Each object represents a separate cluster.

\item {\tt birds}~\cite{klu2022birds} (CC0: Public Domain) contains $224\times 224\times 3$ color images of 525 species of birds. Following prior work~\cite{yu2023pecann,Bateni2025EfficientCentroid}, we pass each image through ConvNeXt~\cite{Liu_2022_CVPR} to obtain an embedding. Each species corresponds to a ground-truth cluster.

\item {\tt alloprof} (CC BY-NC-SA 4.0) is a French educational text dataset curated by AlloProf~\cite{alloprofpaper}. Following the Massive Text Embedding Benchmark (MTEB)~\cite{MTEB}, we cluster documents into $13$ topical classes, using $4096$-dimensional embeddings generated by Qwen3-Embedding-8B~\cite{qwen3embedding}, which ranks at the top of the MTEB multilingual leaderboard among open-weight models (as of February 2026).

\item {\tt hal} (Apache-2.0) is a text dataset built from the HAL open-access repository of scientific papers~\cite{halclustering}. Following MTEB~\cite{MTEB}, we use the clustering split (with $10$ clusters in our setup) and generate $4096$-dimensional embeddings using Qwen3-Embedding-8B~\cite{qwen3embedding}.

\item {\tt news} (CC BY 4.0) is the Twenty Newsgroups dataset by UCI~\cite{UCI}, which consists of posts from $20$ newsgroups. Each newsgroup represents a distinct cluster. We generate $4096$-dimensional embeddings using Qwen3-Embedding-8B~\cite{qwen3embedding}.

\item {\tt cifar-10} and {\tt cifar-100} are standard vision datasets consisting of $32\times 32\times 3$ color images with $10$ and $100$ classes, respectively~\cite{krizhevsky2009learning}. Each class represents a distinct cluster. We generate embeddings by passing each image through CLIP ViT-B/32~\cite{clip}.

\item {\tt fashion-mnist} (MIT License)~\cite{fashion-mnist} consists of $28\times 28$ grayscale images of fashion items from $10$ categories. Each category represents a distinct cluster. We generate embeddings using CLIP ViT-B/32~\cite{clip}.

\item {\tt food-101} (CC BY-SA 4.0)~\cite{food101} contains $101{,}000$ color images from $101$ food categories. Each category represents a distinct cluster. We generate embeddings using CLIP ViT-B/32~\cite{clip}.

\item {\tt emnist} (CC0: Public Domain)~\cite{emnist} is an extension of MNIST provided by NIST. We use the Balanced split, which contains $131{,}600$ character images across $47$ classes. Each class represents a distinct cluster. We generate embeddings using CLIP ViT-B/32~\cite{clip}.

\item {\tt reddit} (CC BY 4.0) is a text dataset from MTEB~\cite{MTEB}, where the goal is to cluster Reddit post titles into subreddits. We use GTE-large embeddings~\cite{li2023generaltextembeddingsmultistage}, as done in prior work~\cite{yu2023pecann}.
\end{itemize}

\subsection{Running Times}\label{app:running_times}
\begin{table}[t]
\centering
\caption{\small Running times (seconds) on \texttt{mnist} and \texttt{birds}. We report results for \texttt{fastcluster}, \texttt{scikit-learn} (which does not include centroid), our optimized C++ implementations (five classical linkages and four chamfer variants), and the $(1+\varepsilon)$-approximate centroid baseline.}
\label{tab:runtime}
\footnotesize
\begin{tabular}{cc|rr}
\toprule
 & Method & \texttt{mnist} & \texttt{birds} \\
\midrule
\multirow{4}{*}{\tt scikit-learn}
  & Average   & 1504.08 & 2688.29 \\
  & Complete  & 1489.73 & 2693.77 \\
  & Single    & 2555.94 & 4778.61 \\
  & Ward's    & 1506.62 & 2711.93 \\
\midrule
\multirow{5}{*}{\tt fastcluster}
  & Average   & 1299.97 & 2444.29 \\
  & Centroid  & 1265.57 & 2420.26 \\
  & Complete  & 1300.08 & 2453.85 \\
  & Single    & 1239.04 & 2386.04 \\
  & Ward's    & 1302.19 & 3465.15 \\
\midrule
\cite{Bateni2025EfficientCentroid} & $\text{Centroid}_{0.1}$ & 82.18   & 79.45 \\
\midrule
\multirow{9}{*}{This paper} 
  & Average         & 442.576 & 526.627 \\
  & Centroid        & 478.971 & 573.355 \\
  & Complete        & 467.352 & 527.011 \\
  & Single          & 420.275 & 514.670 \\
  & Ward's          & 492.320 & 603.124 \\
  & $\Ch$         & 404.205 & 486.380 \\
  & $\Ch_N$       & 400.839 & 492.087 \\
  & $\Ch_S$       & 414.059 & 513.012 \\
  & $\Ch_{NS}$    & 418.027 & 521.356 \\
\bottomrule
\end{tabular}
\end{table}

We report running times on the two of the large datasets where all exact methods complete—{\tt mnist} and {\tt birds}. Table~\ref{tab:runtime} lists results for the baseline classical linkages from {\tt scikit-learn} and {\tt fastcluster}, the approximate Centroid baseline $(\text{Centroid}_{0.1})$~\cite{Bateni2025EfficientCentroid}, and our implementations of the five classical linkages and the four Chamfer variants (note: {\tt scikit-learn} does not include Centroid linkage). Each number is the average over five runs.

Relative to {\tt fastcluster}, our implementations achieve up to $5.75\times$ speedup (geometric mean: $3.6\times$). Against {\tt scikit-learn}, we achieve up to $9.28\times$ speedup (geometric mean: $4.7\times$). The Chamfer variants ($\Ch$, $\Ch_{N}$, $\Ch_{S}$, $\Ch_{NS}$) run in essentially the same time as the classical linkages within the \hacnn framework; we observe no additional overheads from their merge updates.

On the largest datasets ({\tt reddit}, {\tt covertype}), only Chamfer (via our space–time trade-off implementation) and $\text{Centroid}{0.1}$ complete, as all other exact $O(n^2)$-space methods exhaust memory. The running times are 31{,}680 s and 15{,}300 s for Chamfer, versus 672.7 s and 139.5 s for $\text{Centroid}{0.1}$. These high times for Chamfer are expected given the super-quadratic complexity in this setting.

\begin{table*}[ht]
\centering
\vspace{-1em}
\caption{\small The Adjusted Mutual Information (AMI) and Fowlkes-Mallows Index (FMI) scores of the four Chamfer-linkage variants ($\Ch$, $\Ch_N$, $\Ch_S$ and $\Ch_{NS}$) versus baseline classical linkage functions.
The best score for each dataset is shown in \textcolor{Green}{\underline{\textbf{green bold}}}. Additionally, whenever a Chamfer variant outperforms all baselines on a dataset, it is shown in \textcolor{Blue}{\textbf{blue bold}}.
}
\label{tab:quality-evals_other}
\small
\begin{tabular}{cc|ccccc|cccc}
\toprule
\multicolumn{1}{l}{} &
  Dataset &
  Average &
  Centroid &
  Complete &
  Single &
  Ward's &
  $\Ch$ &
  $\Ch_N$ &
  $\Ch_S$ &
  $\Ch_{NS}$\\
\midrule
\multirow{18}{*}{\begin{sideways}AMI\end{sideways}} 
     & {\tt iris}       & 0.732 & \best{0.803} & 0.679 & 0.732 & 0.767 & \best{0.803} & \best{0.803} & 0.76 & 0.78 \\
     & {\tt wine}       & 0.424 & 0.424 & \best{0.436} & 0.39 & 0.424 & 0.424 & 0.424 & 0.425 & 0.413 \\
     & {\tt cancer}     & 0.454 & 0.425 & 0.425 & 0.325 & 0.421 & \best{0.464} & 0.389 & 0.423 & 0.436 \\
     & {\tt digits}     & 0.836 & 0.79 & 0.692 & 0.794 & 0.867 & \best{0.89} & \better{0.888} & 0.825 & 0.817 \\
     & {\tt faces}      & 0.729 & 0.684 & 0.66 & 0.754 & 0.727 & 0.754 & 0.746 & 0.738 & \best{0.762} \\
     & {\tt coil-20}    & 0.777 & 0.784 & 0.73 & \best{0.901} & 0.824 & 0.888 & 0.836 & 0.856 & 0.806 \\
     & {\tt coil-100}   & 0.834 & 0.831 & 0.791 & \best{0.919} & 0.843 & 0.892 & 0.885 & 0.894 & 0.874 \\
     & {\tt alloprof} & \best{0.786} & 0.775 & 0.775 & 0.714 & 0.775 & 0.776 & 0.775 & 0.728 & 0.783 \\ 
     & {\tt usps}       & 0.678 & 0.576 & 0.541 & 0.336 & 0.735 & \best{0.818} & 0.732 & 0.728 & 0.662 \\
     & {\tt hal} & 0.387 & 0.372 & 0.371 & 0.147 & 0.396 & \best{0.411} & 0.359 & 0.369 & 0.365 \\ 
     & {\tt news} & 0.609 & 0.592 & 0.576 & 0.556 & 0.618 & 0.613 & \best{0.627} & 0.594 & 0.604 \\
     & {\tt cifar-10} & 0.699 & 0.43 & 0.565 & 0.149 & 0.746 & \best{0.762} & 0.704 & 0.691 & 0.557 \\ 
     & {\tt cifar-100} & 0.557 & 0.425 & 0.496 & 0.240 & \best{0.569} & 0.566 & 0.515 & 0.536 & 0.507 \\ 
     & {\tt mnist}      & 0.626 & 0.499 & 0.48 & 0.211 & 0.698 & \best{0.853} & \better{0.75} & \better{0.702} & 0.632 \\
     & {\tt fashion-mnist} & 0.628 & 0.534 & 0.495 & 0.221 & 0.668 & \best{0.705} & 0.602 & 0.662 & 0.565 \\ 
     & {\tt birds}      & 0.873 & 0.816 & 0.843 & 0.722 & 0.908 & \best{0.917} & 0.866 & 0.879 & 0.851 \\
     & {\tt food-101} & 0.642 & 0.450 & 0.553 & 0.235 & 0.653 & \best{0.682} & 0.617 & 0.596 & 0.565 \\ 
     & {\tt emnist} & 0.527 & 0.454 & 0.445 & 0.163 & 0.536 & \best{0.640} & \better{0.543} & \better{0.548} & 0.529 \\ 
\cmidrule{2-11}
     & {\bf Geomean}    & 0.639 & 0.570 & 0.571 & 0.389 & 0.657 & \best{0.694} & 0.647 & 0.644 & 0.619 \\
\midrule 
\multirow{18}{*}{\begin{sideways}FMI\end{sideways}} 
     & {\tt iris}       & 0.771 & 0.841 & 0.769 & 0.821 & 0.822 & 0.841 & \best{0.847} & 0.814 & 0.83\\
     & {\tt wine}       & 0.645 & 0.645 & 0.637 & 0.617 & 0.645 & 0.645 & 0.645 & 0.645 & \best{0.67} \\
     & {\tt cancer}     & \best{0.797} & 0.776 & 0.777 & 0.78 & 0.739 & 0.792 & 0.761 & 0.77 & 0.783 \\
     & {\tt digits}     & 0.825 & 0.753 & 0.577 & 0.812 & 0.86 & \better{0.887} & \best{0.888} & 0.786 & 0.815 \\
     & {\tt faces}      & 0.615 & 0.583 & 0.531 & 0.661 & 0.629 & \better{0.664} & 0.64 & 0.637 & \best{0.665} \\
     & {\tt coil-20}    & 0.623 & 0.656 & 0.591 & \best{0.859} & 0.716 & 0.792 & 0.689 & 0.778 & 0.668 \\
     & {\tt coil-100}   & 0.644 & 0.659 & 0.605 & \best{0.8} & 0.682 & 0.746 & 0.715 & 0.764 & 0.691 \\
     & {\tt alloprof} & \best{0.832} & 0.811 & 0.778 & 0.768 & 0.817 & 0.785 & 0.818 & 0.782 & 0.818 \\
     & {\tt usps}       & 0.619 & 0.553 & 0.445 & 0.488 & 0.71 & \best{0.816} & \better{0.753} & \better{0.725} & 0.641 \\
     & {\tt hal} & 0.467 & 0.453 & 0.390 & 0.390 & 0.509 & \best{0.543} & 0.390 & 0.423 & 0.440 \\
     & {\tt news} & \best{0.499} & 0.427 & 0.449 & 0.348 & 0.491 & 0.496 & 0.486 & 0.476 & 0.483 \\
     & {\tt cifar-10} & 0.732 & 0.441 & 0.445 & 0.316 & 0.720 & 0.722 & \best{0.740} & 0.628 & 0.594 \\
     & {\tt cifar-100} & 0.323 & 0.2 & 0.223 & 0.154 & 0.311 & \better{0.325} & \best{0.334} & 0.283 & 0.277 \\
     & {\tt mnist}      & 0.538 & 0.393 & 0.317 & 0.358 & 0.634 & \best{0.834} & \better{0.806} & \better{0.652} & 0.607 \\
     & {\tt fashion-mnist} & 0.573 & 0.513 & 0.368 & 0.328 & 0.603 & \best{0.649} & \better{0.615} & \better{0.605} & 0.526 \\
     & {\tt birds}      & 0.65 & 0.593 & 0.59 & 0.502 & \best{0.752} & 0.75 & 0.635 & 0.692 & 0.613 \\
     & {\tt food-101} & 0.483 & 0.218 & 0.275 & 0.182 & 0.457 & \best{0.484} & 0.478 & 0.387 & 0.359 \\
     & {\tt emnist} & 0.274 & 0.224 & 0.146 & 0.146 & 0.28 & \best{0.429} & \better{0.405} & \better{0.334} & \better{0.307} \\
\cmidrule{2-11}
     & {\bf Geomean}    & 0.583 & 0.498 & 0.455 & 0.450 & 0.607 & \best{0.656} & \better{0.624} & 0.595 & 0.572\\
\bottomrule
\end{tabular}
\end{table*}

\begin{table}[h]
\centering
\vspace{-1em}
\caption{\small The Adjusted Mutual Information (AMI) and the Fowlkes-Mallows Index (FMI) scores of chamfer- and approximate centroid-linkage on {\tt reddit} and {\tt covertype}. The best quality score for each dataset is in bold and underlined.
}
\label{tab:large-evals_other}
\vspace{-0.5em}
\small
\begin{tabular}{cc|c|c}
\toprule
\multicolumn{1}{l}{} &
  Dataset &
  $\text{Centroid}_{0.1}$ &
  $\Ch$ \\
\midrule
\multirow{2}{*}{\begin{sideways}AMI\end{sideways}} 
     & {\tt reddit}       & 0.426 & \best{0.521}\\
     & {\tt covertype}       & 0.163 & \best{0.18}\\
\midrule 
\multirow{2}{*}{\begin{sideways}FMI\end{sideways}} 
     & {\tt reddit}       & 0.144 & \best{0.396} \\
     & {\tt covertype}       & \best{0.614} & \best{0.614}\\
\bottomrule
\end{tabular}
\vspace{-1.5em}
\end{table}

\subsection{Quality Evaluation}\label{app:quality_evals}
\myparagraph{Evaluation Metrics} 
We evaluate clustering quality against ground-truth labels using the standard metrics:
\emph{Adjusted Rand Index (ARI)}, \emph{Normalized Mutual Information (NMI)},
\emph{Adjusted Mutual Information (AMI)}, and \emph{Fowlkes--Mallows Index (FMI)}.
Given a dendrogram $T$ with merge costs $w(\cdot)$ on internal nodes, we report, for each method, the \emph{best} score attained over the family of clusterings obtained along the following three merge orders:

\begin{enumerate}
  \item \textbf{Merge-order.} The prefix sequence of clusterings produced by the algorithm itself (from $n$ singletons down to one cluster). These are the clusterings after each merge.

  \item \textbf{Least-available-merge-first.} Starting from all singletons, repeatedly expose the currently smallest-cost merge in the dendrogram, independent of the algorithm’s original merge order; again, we evaluate the clustering after each exposed merge.

  \item \textbf{Monotonicized least-available-merge-first.} Some linkages (e.g., Centroid and Chamfer variants) are non-monotone, so $w$ can decrease along a root-to-leaf path. We ``monotonicize'' by assigning each internal node
  $$
  \widetilde{w}(u)=\sum_{v\in \mathrm{subtree}(u)} w(v),
  $$
  which strictly increases toward the root. We then apply the least-available-merge-first traversal using $\widetilde{w}$ and evaluate the clustering after each exposed merge.
\end{enumerate}

For ARI, NMI, and FMI, the reported score is the \emph{maximum} over all clusterings visited along these three orders. We compute these scores incrementally, starting with the trivial clustering of all singletons, to finally a single cluster with everything, updating the scores as we proceed. This takes $O(n)$ time in total, per ordering.

AMI calculations are more complex: it requires the complex expected mutual information term, which prevents the same constant-time incremental update. For AMI, we therefore, evaluate \emph{thresholded cuts} using a geometric schedule of thresholds (merge costs spaced multiplicatively), which closely approximates the maximum while keeping the total time near-linear. We provide implementations of these optimized evaluation metrics along with the rest of the code in the supplementary material.

\myparagraph{Results} We described the ARI and NMI results in detail in Section~\ref{sec:experiments}. Table~\ref{tab:quality-evals_other} summarizes the results for AMI and FMI scores. For the large datasets, the AMI and FMI scores are presented in Table~\ref{tab:large-evals_other}.

\end{document}